\def\year{2020}\relax
\documentclass[letterpaper]{article} 
\usepackage{aaai20}  
\usepackage{times}  
\usepackage{helvet} 
\usepackage{courier}  
\usepackage[hyphens]{url}  
\usepackage{graphicx} 
\usepackage{graphicx}  

\usepackage{amsthm}
\usepackage{amsmath}
\usepackage{amsfonts}
\usepackage[ruled,linesnumbered]{algorithm2e}
\usepackage{xcolor}
\usepackage{bm}
\newtheorem{theorem}{Theorem}
\newtheorem{definition}{Definition}
\newtheorem{lemma}{Lemma}
\newtheorem{fact}{Fact}

\newtheorem{assumption}{Assumption}

\newcommand{\citet}[1]{\citeauthor{#1} \shortcite{#1}} 
\newcommand{\citep}{\cite}

\title{Incentivized Exploration for Multi-Armed Bandits under Reward Drift}

\author{
Zhiyuan Liu \thanks{Equal Contribution}\\
Department of Computer Science\\
University of Colorado, Boulder\\
zhiyuan.liu@colorado.edu	
\And 
Huazheng Wang \footnotemark[1] \thanks{Huazheng Wang is supported by Bloomberg Data Science Ph.D. Fellowship. }\\
Department of Computer Science\\
University of Virginia\\
hw7ww@virginia.edu	
\And
Fan Shen\\
Technology, Cybersecurity and Policy\\
University of Colorado, Boulder\\
fan.shen@colorado.edu
\AND
Kai Liu\\
Computer Science Division\\
Clemson University\\
kail@clemson.edu
\And
Lijun Chen\\
Department of Computer Science\\
University of Colorado, Boulder\\
lijun.chen@colorado.edu
}

\newcommand{\compilehidecomments}{false}
\ifthenelse{ \equal{\compilehidecomments}{true} }{%
	\newcommand{\huazheng}[1]{}
}{
	\newcommand{\huazheng}[1]{{\color{blue!50!black}  [\text{Huazheng:} #1]}}
}

\allowdisplaybreaks

\begin{document}
	\maketitle
	\begin{abstract}
		We study incentivized exploration for the multi-armed bandit (MAB) problem where the players receive compensation for exploring arms other than the greedy choice and may provide biased feedback on reward. We seek to understand the impact of this drifted reward feedback by analyzing the performance of three instantiations of the incentivized MAB algorithm: UCB, $\varepsilon$-Greedy, and Thompson Sampling. Our results show that they all achieve $\mathcal{O}(\log T)$ regret and compensation under the drifted reward, and are therefore effective in incentivizing exploration. Numerical examples are provided to complement the theoretical analysis.

	\end{abstract}
	
	\section{Introduction}

Multi-armed bandit (MAB) problem is a classical model for sequential decision making under uncertainty, and 
finds applications in many real world systems such as recommender systems \cite{li2010contextual,bouneffouf2012contextual}, search engine systems \cite{radlinski2008learning,yue2009interactively}  and cognitive radio networks \cite{gai2010learning}, to just name a few. In the traditional MAB model, the decision maker (the principal) who selects the arm to pull and the action performers (the players) who actually pull the arm are assumed to be the same entity. 
This is, however, not true for several important real world applications where the principal and players are different entities with different interests. Take the Amazon product rating as an example: Amazon (the principal) would like the customers (the players) to buy and try different products (arms) of certain type in order to identify the best product (i.e., exploration), while the customers are heavily influenced by the current ratings on the products and behave myopically, i.e., select the product that currently has the highest rating (i.e., exploitation). It is well known that such exploitation-only behavior can be far from the optimal \cite{bubeck2012regret,sutton2018reinforcement}.  In the traditional MAB setting, the principal, who is also the player, strives to find the optimal tradeoff between exploration and exploitation and execute it accordingly. When the principal and players are different entities, misaligned interests between them need to be reconciled in order to balance exploration and exploitation in an optimal manner. 

Incentivized learning has been proposed for the MAB problem to reconcile different interests between the principal and the players  \cite{frazier2014incentivizing,mansour2015bayesian,Wang2018MultiarmedBW,immorlica2019bayesian}. In order to incentivize exploration, the principal provides certain compensation to the player so that s/he will pull the arm other than the greedy choice that currently has the best empirical reward. The goal of the principal is to maximize the cumulative reward while minimizing the total compensation to the players. 


However, existing incentivized MAB models \cite{han2015incentivizing,Wang2018MultiarmedBW,immorlica2018incentivizing,liu2018incentivizing,hirnschall2018learning} assume that the players provide unbiased stochastic feedback on reward\footnote{We will use reward and feedback interchangeably in this paper}
even after they receive certain incentive from the principal. This assumption does not always hold in the real world scenarios: work based on industrial level experiments in \cite{martensen2000drivers,razak2016impact,ehsani2015effect} shows that the customers are inclined to give higher evaluation (i.e., increased reward) with incentive such as discount and coupon. The compensation could even be the primary driver of customer satisfaction \cite{martensen2000drivers,lee2005customer}. This drift in reward feedback may cause negative impact on the exploration and exploitation tradeoff, e.g., a suboptimal arm is mistaken as the optimal one because of the incentives and the players will keep pulling it even after the compensation is removed. This has been ignored in previous research.

In this paper, we aim to investigate the impact of drifted reward feedback in the incentivized MAB problem. Specifically, we consider a general incentivized exploration algorithm where the player receives a compensation that is the difference in reward between the principal's choice and the greedy choice, and provides biased feedback that is the sum of the true reward of an arm and a drift term that is a non-decreasing function of the compensation received for pulling this arm. 
We seek to answer the important question if the compensation scheme is effective in incentivizing exploration under reward drift from two intertwining aspects:  (1) if the algorithm is robust to drifted reward so that the sequential decisions based on biased feedback still enjoy a small regret, and (2) if the proposed incentive mechanism is cost-efficient to the principal. We analyze the regret and compensation for three instantiations of the algorithm where the principal employs UCB, $\epsilon$-Greedy, and Thompson Sampling, respectively.  Our analytical results, complemented by
numerical experiments, show that with drifted reward the proposed compensation scheme achieves both $\mathcal{O}(\log T)$ regret and  $\mathcal{O}(\log T)$ compensation, and is thus effective in incentivizing exploration. 

	\subsection{Related Work}
Incentivized learning has attracted a lot of attention since the work \cite{frazier2014incentivizing}. In \cite{frazier2014incentivizing}, the authors proposed a Bayesian incentivized model with discounted regret and compensation, 
and characterized the relationship between the reward, compensation, and discount factor. In \cite{mansour2015bayesian}, the authors studied the non-discount case and proposed an algorithm that has $\mathcal{O}(\sqrt{T})$ regret. In \cite{Wang2018MultiarmedBW}, the authors analyzed the non-Bayesian and non-discount reward case and showed $\mathcal{O}(\log T)$ regret and compensation for incentivized exploration based on simplified MAB algorithms. But all the models and analysis are under the assumption that the players' feedbacks are unbiased under compensation. In contrast, we consider biased feedback under compensation, and show that the incentivized exploration with reward drift can still achieve  $\mathcal{O}(\log T)$ regret and compensation. 


Related work also includes those in robustness of MAB under adversarial attack. In \cite{lykouris2018stochastic}, the authors proposed a multi-layer active arm elimination race algorithm for stochastic bandits with adversarial corruptions whose performance degrades linearly to the amount of corruptions. In \cite{feng2019intrinsic}, the authors studied strategic behavior of rational arms and show that UCB, $\epsilon$-Greedy, and Thompson sampling achieve $\mathcal{O}(\max \{ B, \log T\})$ regret bound  under any strategy of the strategic arms, where $B$ is the total budget across arms. On the other hand, in \cite{jun2018adversarial}, the authors constructed attacks by decreasing the reward of non-target arms, and showed that their algorithm can trick UCB and $\epsilon$-Greedy to pull the optimal arm only $o(T)$ times under an $\mathcal{O}(\log T)$ attack budget. All the modeled attacks are from exogenous sources, e.g., malicious users, while in our paper, the reward drift can have an interpretation as arising from attacks but generated endogenously by the incentivized exploration algorithm itself. 
	\section{Model, Notation, and Algorithm}

Consider a variant of the multi-armed bandit problem where a principal has $K$ arms, denoted by the set $[K] = \{1,\cdots,K\}$. The reward of each arm $i \in [K]$ follows a distribution with support $[0,1]$ and mean $\mu_{i}$ that is unknown. Without loss of generality, we assume that arm $1$ is the unique optimum with the maximum mean. Denote by $\Delta_i = \mu_1 - \mu_i$ the reward gap between arm $1$ and arm $i (i \neq 1)$, and let $\Delta = \min_i \Delta_i$. At each time $t = 1, \cdots, T$, a new player will pull one arm $I_t\in [K]$ and receive a reward $r_t$ that will fed back to the principal and other players. Let $n_i(t) = \sum_{\tau =1}^{t-1} \mathbb{I}(I_\tau = i)$ denote the number of times that arm $i$ has been pulled up to time $t$ and $\hat{\mu}_i(t) =\frac{1}{n_i(t)} \sum_{\tau=1}^{t-1} r_\tau\mathbb{I}(I_\tau = i)$ the corresponding empirical average reward, where the indicator function $\mathbb{I}(A)=1$ if $A$ is true and $\mathbb{I}(A)=0$ otherwise. 

In real world applications, the principal and players may exhibit different behaviors. The principal would like to see the players select the best arm and maximize the cumulative reward. On the other hand, the players may be heavily influenced by other players' feedback, e.g., the reward history of the arms, and behave myopically, i.e., pull the arm that currently achieves the highest empirical reward (exploitation). It is well known that such a myopic exploitation-only behavior can be far from  the optimum due to the lack of exploration \cite{sutton2018reinforcement}. The principal cannot pull the arm directly, but can provide certain compensation to incentivize the players to pull arms with suboptimal empirical reward (exploration). However, this compensation may affect the players' feedback \cite{martensen2000drivers}, which results in a biased reward history and disturbs both the principal and players' decisions. Specifically, we assume that at time $t$ there is a drift $b_t$ in feedback that is caused by compensation $x$, captured by an unknown function $b_t=f_t (x)$ with the following properties. 

\begin{assumption} \label{assumption:1}
	The reward drift function $f_t(x)$ is non-decreasing with $f_t(0) = 0$, and is Lipschitz continuous, i.e., for any $x$ and $y$, there exists a constant $l_t$ such that
	\begin{align}
	| f_t(x) - f_t(y) | \leq l_t |x-y|. 
	\end{align}
\end{assumption}
The biased feedback $r_t + b_t$  is then collected, and the principal and players know only the sum and cannot distinguish each part.

Let $l = \max_{t} l_t$ for later use.  Denote by $E_t^i=1$ the event that arm $i$ is pulled with compensation at time $t$ and $E_t^i = 0$ otherwise.  Denote  $B_i(t) = \sum_{\tau =1}^{t-1} b_{\tau} \mathbb{I}(E_{\tau}^{i} = 1)$ be the cumulative drift of arm $i$ up to time $t$ and $\bar{\mu}_i(t) = \hat{\mu}_i(t) + \frac{B_i(t)}{n_i(t)}$ be the corresponding average drifted reward. The general incentive mechanism and algorithm are described in Algorithm \ref{algo:1}. 
\begin{algorithm}
	\caption{Incentivized MAB under Reward Drift}\label{algo:1}
	\For{$t = 1,2,3,\cdots,T$}{
		The principal selects arm $I_t$ according to certain algorithm;\\
		The player will choose $G_t = \arg\max \bar{\mu}_i(t)$ w/o compensation; \\
		\uIf{$G_t = I_t$}
		{
			The principal does not provide compensation; \\
			The player pulls arm $I_t$ and receives reward $r_t$;
		}
		\Else{
			The principal provides compensation $\bar{\mu}_{G_t} - \bar{\mu}_{I_t}$; \\
			The player pulls arm $I_t$ and receives reward $r_t + b_t$, with $b_t = f_t(\bar{\mu}_{G_t} - \bar{\mu}_{I_t})$. 
		}
	Update average reward $\bar{\mu}_{I_t}$
	}
	
\end{algorithm}

We characterize the performance of the incentivized exploration algorithm in terms of two metrics -- the expected cumulative regret that quantifies the total loss because of not pulling the best arm,  and the cumulative compensation that the principal pays for incentivizing exploration: 
\begin{align}
\mathbb{E}(R(T)) &= \mathbb{E}(\sum_{t=1}^{T}(\mu_1 - \mu_{I_t})) = \sum_{i=2}^{N} \Delta_i \mathbb{E}(n_i(T+1)),  \nonumber\\
\mathbb{E}(C(T)) &= \mathbb{E}(\sum_{t=1}^{T}(\bar{\mu}_{I_t} - \bar{\mu}_{G_t})). \nonumber  
\end{align}
Notice that in Algorithm~1 the compensation and pulled arm are decided based on biased feedback which may not be an accurate reflection of an arm's reward, while the regret is in terms of the ``true'' reward that is unknown. We seek to answer the important question if the proposed compensation scheme is effective in incentivizing exploration from two intertwining aspects: (1) if the algorithm is robust to drifted reward so that the sequential decisions based on biased feedback still enjoy a sublinear regret, and (2) if the proposed incentive mechanism is cost-efficient to the principal. 
There are different arm selection strategies that the principal can employ, i.e., Step 2 of Algorithm 1. In the next section, in order to answer the above question, we analyze the cumulative regret and compensation under several typical multi-armed bandit algorithms such as UCB, $\epsilon$-Greedy, and Thompson Sampling.  


\section{Regret and Compensation Analysis}
In this section, we consider three instantiations of Algorithm 1 when the principal employs UCB, $\epsilon$-Greedy, and Thompson Sampling at Step 2, respectively.    As will be seen later, our analysis shows that the proposed compensation scheme is effective in incentivizing exploration under reward drift. 


\subsection{UCB policy}

Consider first the case where the principal applies the UCB policy, i.e., uses the sum  of average biased reward and upper confidence bound $\bar{\mu}_i(t) +  \sqrt{\frac{2\log t}{n_i(t)}}$ as the criterion to choose the arm to explore, as shown in Algorithm \ref{algo:2}. The main result is summarized in Theorem \ref{theroem:1}. \begin{algorithm}
	\caption{Incentivized UCB under Reward Drift}\label{algo:2}
	\For{$t = 1,2,3,\cdots,T$}{
		The principal selects arm $I_t = \arg\ \max_i\ \bar{\mu}_i(t) + \sqrt{\frac{2\log t}{n_i(t)}}  $;\\
		Steps 3-11 of Algorithm 1.
	}
\end{algorithm}
\begin{theorem}\label{theroem:1}
	For the incentivized UCB algorithm, the expected regret $R(T)$ and compensation $C(T)$ are bounded as follows: 
	\begin{align}
	\mathbb{E}(R(T)) &\leq  \sum_{i= 2}^{N} \frac{8(l\!+\!1)^2\log T}{\Delta_i} \!+\! \frac{\Delta_i(K\!-\!1)\pi^2}{3}, \\
	\mathbb{E}(C(T)) &\leq   \sum_{i=2}^{N} \frac{16(l\!+\!1) \log T}{\Delta_i} \! +\!  \frac{16(l\!+\!1) \log T}{\Delta} \nonumber \\
	&~~~~~~\!+\! 2\pi K \sqrt{\frac{2\log T}{3}}. 
	\end{align} 
\end{theorem}
\begin{proof}
	Notice that compensation is incurred under the conditions:
	\begin{equation*}
	\begin{aligned}
	\bar{\mu}_{I_t}(t) &\leq \bar{\mu}_{G_t}(t), \\
	\bar{\mu}_{I_t}(t) + \sqrt{\frac{2\log t}{n_{I_t}(t)}}	&\geq \bar{\mu}_{G_t}(t) + \sqrt{\frac{2 \log t}{n_{G_t}(t)}}.
	\end{aligned}
	\end{equation*}
	By the second condition, 
the compensation 
\begin{equation}
\bar{\mu}_{G_t}(t) - \bar{\mu}_{I_t}(t) \leq \sqrt{\frac{2\log t}{n_{I_t}(t)}}, \label{eq:cb}
\end{equation}
and further by Assumption \ref{assumption:1}, the drift $b_t \leq l_t\sqrt{\frac{2\log t}{n_{I_t}(t)}}.$
The total drift $B_i(t)$ of arm $i$ can be bounded as follows (due to space limit, the details of inequality \eqref{bound:B} are provided in supplementary material): 
    \begin{align}
	B_i(t) &= \sum_{\tau =1}^{t} b_\tau \mathbb{I}(E_\tau^i = 1)
	 \leq 2l \sqrt{2n_i(t) \log t}. \label{bound:B}
	\end{align}
	For each sub-optimal arm $i \neq 1$, if this arm is pulled by the player at $t$ (with or without compensation), it must hold that 
	\begin{equation*}
	\hat{\mu}_{i}(t) + \frac{B_i(t)}{n_i(t)} + \sqrt{\frac{2\log t}{n_i(t)}} \geq \hat{\mu}_{1}(t) + \frac{B_1(t)}{n_1(t)} + \sqrt{\frac{2\log t}{n_1(t)}}.
	\end{equation*}
	So, the probability that arm $i$ is pulled by the player at time $t$ can be bounded by the following:
	\small
	\begin{equation*}
	\begin{aligned}
	&\Pr(I_t=i)\\
	\leq &\Pr\left(\!\!\hat{\mu}_{i}(t) \!+\! \frac{B_i(t)}{n_i(t)} \!+\! \sqrt{\frac{2\log t}{n_i(t)}} \!\geq\! \hat{\mu}_{1}(t)\! +\! \frac{B_1(t)}{n_1(t)} \!+\! \sqrt{\frac{2\log t}{n_1(t)}}\right) \\
	\leq& \Pr\left(\!\hat{\mu}_{i}(t) \!+\! (2l\!+\!1)\sqrt{\frac{2\log t}{n_i(t)}} \!\geq\! \hat{\mu}_{1}(t)\! + \!\frac{B_1(t)}{n_1(t)} \!+\! \sqrt{\frac{2\log t}{n_1(t)}}\right) \\
	\leq& \Pr\left(\!\hat{\mu}_{i}(t) \!+\! (2l\!+\!1)\sqrt{\frac{2\log t}{n_i(t)}} \!\geq\! \hat{\mu}_{1}(t) \!+\! \sqrt{\frac{2\log t}{n_1(t)}}\right), 
	\end{aligned}
	\end{equation*}
	\normalsize
	where the second inequality is due to the bound \eqref{bound:B} on cumulative drift. Similar to the analysis in \cite{auer2002finite}, notice that if the event $\!\hat{\mu}_{i}(t) \!+\! (2l\!+\!1)\sqrt{\frac{2\log t}{n_i(t)}} \!\geq\! \hat{\mu}_{1}(t) \!+\! \sqrt{\frac{2\log t}{n_1(t)}}$ happens, 
	 one of the following three events must happen:
	\small
	\begin{align*}
	X_i(t) &: {\hat{\mu}_i(t) \geq \mu_i + \sqrt{\frac{2\log t}{n_i(t)} }}, \\
	Y_1(t) &: {\hat{\mu}_1(t) \leq  \mu_1 - \sqrt{\frac{2\log t}{n_1(t)} }}, \\
	Z_i(t) &: {2(l+1)\sqrt{\frac{2\log t}{n_i(t)}}} \geq  \Delta_i.
	\end{align*}
	\normalsize
	Therefore, $\Pr(I_t=i) \leq \Pr (X_i(t))+ \Pr (Y_1(t))+\Pr (Z_i(t))$. 
	By the Chernoff-Hoeffding's inequality \cite{hoeffding1994probability}, 
	\begin{align*}
	\Pr(X_i(t)) \leq \frac{1}{t^2},~~~~~~\Pr(Y_1(t)) \leq \frac{1}{t^2},
	\end{align*}
	and their sum from $t=1$ to $T$ is bounded by $\frac{\pi^2}{3}$. If $n_i(t) \geq \frac{8(l+1)^2 \log T}{\Delta_i^2}$, the event $Z_i(t)$ will not happen, and thus  $\sum_{t=1}^{T} \Pr (Z_i(t))\leq \frac{8(l+1)^2 \log T}{\Delta_i^2}$.  We can bound $\mathbb{E}[n_i(T)]$ as follows: 
	\begin{align*}
	\mathbb{E}(n_i(T))& 
	= \sum_{t=1}^{T} \Pr (I_t = i)\\
	&\leq \sum_{t=1}^{T} \left(\Pr (X_i(t))+ \Pr (Y_1(t))+\Pr (Z_i(t))\right)\\
	& \leq  \frac{8(l+1)^2 \log T}{\Delta_i^2}+ \frac{\pi^2}{3}.
	\end{align*} 
	So, the expected regret 
	$$\mathbb{E}(R(T)) \leq  \sum_{i= 2}^{N} \frac{8(l+1)^2\log T}{\Delta_i} + \frac{\Delta_i(K-1)\pi^2}{3}.$$
	
	The calculation of compensation is a bit different from that of regret since compensation can be incurred even if the best arm is pulled. The player will be compensated to pull arm 1 only when  
	$$\bar{\mu}_1(t) \leq  \bar{\mu}_i(t),$$
	$$\bar{\mu}_1(t) + \sqrt{\frac{2\log t}{n_1(t)}} \geq \bar{\mu}_i(t) + \sqrt{\frac{2\log t}{n_i(t)}},$$ 
which requires $n_1(t) \leq n_i(t)$. 
So, the average number of times when the players are compensated to pull arm $1$ is smaller than $\max_{i\neq 1} \mathbb{E}(n_i(T))$. 
Denote by $C_i(t)$ the total compensation the players have received to pull arm $i$ up to time $t$. Recall the bound \eqref{eq:cb}, we can bound the total compensation as follows: 
	\begin{align}
	\mathbb{E}(C(T)) &= \mathbb{E}\left(C_1(T) + \sum_{i=2}^{K}C_i(T)\right) \nonumber\\
	& \leq  \sum_{m =1 }^{\max_{i \neq 1} \mathbb{E}(n_i(T))} \!\!\!\!\!\!\sqrt{\frac{2\log T}{m}} + \sum_{i=2}^{K} \sum_{m = 1}^{\mathbb{E}(n_i(T))}\!\!\!\!\!\sqrt{\frac{2\log T}{m}} \nonumber\\
	& \leq \frac{16(l+1) \log T}{\Delta} + 2K\pi \sqrt{\frac{2\log T}{3}} \nonumber\\\
	&~~~~~+ \sum_{i=2}^{K} \frac{16(l+1) \log T}{\Delta_i}. \nonumber
	\end{align}        
\end{proof}  


\subsection{$\varepsilon$-Greedy policy}

We now consider the case where the principal uses the $\varepsilon$-Greedy  policy as shown in Algorithm \ref{algo:3}, with the choice of exploration probability $\varepsilon_t$ from  that shows diminishing $\varepsilon_t$ achieves better performance. Algorithm \ref{algo:3} involves a random exploration phase (Step 3), and its analysis is more involved. Recall that the ``true'' reward has a normalized support of $[0, 1]$, we therefore assume that the drifted reward $r_t+b_t$ is projected onto $[0,1]$. This assumption is also consistent with real world applications such as  Amazon and Yelp as  their rating systems usually have lower and upper bounds. 

\begin{algorithm}
	\caption{Incentivized $\varepsilon$-Greedy under Reward Drift}\label{algo:3}
	\For{$t = 1,2,3,\cdots,T$}{
		Let $\varepsilon_t = \min(1,\frac{cK}{t})$;\\
		With probability $1-\varepsilon_t$, the principal chooses arm $I_t = \arg\ \max_i\ \bar{\mu}_i(t)$; with probability $\varepsilon_t$, the principal uniformly selects an arm $I_t \in [K]$; \\
		Steps 3-11 of Algorithm 1 with $[r_t + b_t]_0^1$ where $[\cdot]^1_0$ denotes the projection onto $[0, 1]$
	} 
\end{algorithm}
\begin{theorem} \label{theroem:2}
	For the incentivized $\varepsilon$-Greedy algorithm with $\varepsilon_t = \min \{1,\frac{cK}{t}\}$ and $ c \geq \frac{36}{\Delta}$, with a high probability the expected regret $R(T)$ and compensation $C(T)$ are bounded as follows: 
	\begin{align}
	\mathbb{E}(R(T)) &\leq  \sum_{i=2}^{K} cS_i(l)(\log T +1) + c(K\!-\!1)(K+\frac{\pi^2}{6}), \label{eq:egrb}  \\
	\mathbb{E}(C(T)) &\leq  \max(l,1)(c+\sqrt{3c})K\log T,
	\end{align} 
	where $S_i(l) = 1.5 + 3(1+\sqrt{3/c})l + 18c/\Delta_i^2$.
\end{theorem}
\begin{proof}
	Since the biased feedback lies in the interval $[0,1]$, the drift $b_t \leq l_t(\bar{\mu}_{G_t} - \bar{\mu}_{I_t}) \leq  l_t.$ 
A compensation for pulling arm $i$ will be incurred only when the arm is chosen by the principal to explore. By Lemma 2 \cite{agarwal2014taming} in supplementary material, the number of explorations that arm $i$ can receive up to time $t$ is bounded by
	\begin{align}
	m_i(t) \leq c(\log t + 1) + \sqrt{3c\log \frac{K}{\delta}(\log t +1)} \label{eq:neb} 
	\end{align}
	with a probability of at least $1-\delta$. When $t$ is large enough such that $\log t \geq \log \frac{K}{\delta} - 1$, the right hand side of \eqref{eq:neb} is upper bounded by 
	\begin{align*}
	\overline{m}_i(t) = (c+\sqrt{3c})(\log t +1), 
	\end{align*} 
	and the total drift  $B_i(t)$ on arm $i$ up to time $t$ is upper bounded by $l\overline{m}_i(t)$ with a probability of at least $1-\delta$. 
	
	Let $L = \frac{3l\overline{m}_i(T)}{\Delta_i}$ that is chosen to facilitate the analysis. We can bound $\mathbb{E}(n_i(T))$ as follows: 
	\small
	\begin{align*}
&\mathbb{E}(n_i(T)) \\
\leq& \sum_{t=1}^{T} \frac{\varepsilon_t}{K} + \mathbb{E}\left(\sum_{t=1}^{T} (1-\varepsilon_t) \mathbb{I}(I_t = i, n_i(t) \leq L)\right) \\
&+ \mathbb{E}\left(\sum_{t=1}^{T} (1-\varepsilon_t) \mathbb{I}(I_t = i, n_i(t) \geq L)\right) \\
 \leq &  \underbrace{\sum_{t=1}^{T} \frac{\varepsilon_t}{K} + L}_{A} + \mathbb{E}\left(\sum_{t=1}^{T}\! \mathbb{I}(I_t \!=\! i, n_i(t) \!\geq\! L)\!\right) \\
\leq  & A \!+\! \sum_{t=1}^{T} \Pr\!\left( \!\hat{\mu}_i(t)\! +\! \frac{B_i(t)}{n_i(t)} \!\geq\! \hat{\mu}_1(t) \!+\! \frac{B_1(t)}{n_1(t)}\!,\! n_i(t) \!\geq \!L\!\!\right) \\
\leq & A +  \sum_{t=1}^{T} \Pr\left( \hat{\mu}_i(t) +  \frac{\Delta_i}{3} \geq \hat{\mu}_1(t)\right) \\
\leq & A \!+\! \sum_{t=1}^{T}\!\Pr\!\left(\! \hat{\mu}_i(t)  \!\geq \!\mu_i \!+ \! \frac{\Delta_i}{3}\!\right) \!+\! \sum_{t=1}^{T}\!\Pr\!\left(\! \hat{\mu}_1(t)  \!\leq\!  \mu_1 \!- \!\frac{\Delta_i}{3}\!\!\right)\!,
\end{align*}  
\normalsize
	where the second last inequality is due to
	\begin{align*}
	\frac{B_i(t)}{n_i(t)} \leq \frac{g\overline{m}_i(t)}{3g\overline{m}_i(T)/\Delta_i} \leq  \frac{\Delta_i}{3},
	\end{align*}
	and the last inequality uses the fact that $\mu_i=\mu_1-\Delta_i$. 
	By Lemma 3 in supplementary material, when $c \geq \frac{36}{\Delta_i}$, we have
	\begin{align*}
	&~~~~~\sum_{t=1}^{T}\!\Pr\!\left(\! \hat{\mu}_i(t)  \!\geq \!\mu_i \!+ \! \frac{\Delta_i}{3}\!\right) \!+\! \sum_{t=1}^{T}\!\Pr\!\left(\! \hat{\mu}_1(t)  \!\leq\!  \mu_1 \!- \!\frac{\Delta_i}{3}\!\!\right)\! \\
	& \leq (\frac{c}{2} + \frac{18}{\Delta_i^2}) \log T + c(K+\frac{\pi^2}{\Delta_i^2}) + \frac{18}{\Delta_i^2}.
	\end{align*}
	We can also show that $A \!\leq \!c(1 \!+\! 3g(1\!+ \!\sqrt{3/c}))(\log T \!+ \!1)$, and further obtain the bound \eqref{eq:egrb} on expected regret after some straightforward mathematical manipulations. 
	
	For the compensation analysis, notice again that the drifted reward is in $[0, 1]$, so the compensation at each time is less than $1$ and the total compensation the players receive to pull arm $i$ is bounded by the bound $\overline{m}_i$ on the number of explorations it receives. To be consistent with the case with no drift, we write the bound on expected compensation as 
	\begin{align*}
	\mathbb{E}(C(T)) &\leq  \max(l,1)(c+\sqrt{3c})K(\log T + 1).
	\end{align*}

	
\end{proof}

\subsection{Thompson Sampling}

Consider now the case where the principal uses Thompson Sampling as shown in Algorithm \ref{algo:4}.  Thompson Sampling starts with a (prior) distribution on  each arm's reward, and updates the distribution after the arm being pulled. At each time, the principal samples the reward of each arm according to its posterior distribution, and then selects the arm with the highest sample reward.  In this paper, we consider Gaussian prior adopted from \cite{agrawal2013further} since the often used Beta priors are usually for binary reward feedback.

Before we analyze the performance of Algorithm \ref{algo:4}, we first introduce some definitions and notations that are adopted from \cite{agrawal2017near,agrawal2013further}.
\begin{algorithm}
	\caption{Incentivized Thompson Sampling under Reward Drift}\label{algo:4}
	\For{$t = 1,2,3,\cdots,T$}{
		The principal independently samples $\theta_i(t)$ from distribution $\mathcal{N}(\bar{\mu}_i(t),\frac{1}{n_i(t)+1})$ and selects arm $I_t = \arg\ \max_i\ \theta_i(t)$;\\
		Steps 3-11 of Algorithm 1. 
	} 
\end{algorithm}
\begin{definition}
	For each arm $i$, we denote two thresholds $x_i$ and $y_i$ such that $\mu_{i} \leq x_i \leq y_i \leq \mu_{1}$. $E_i^{\mu}(t)$ denotes the event $\bar{\mu}_i(t) \leq x_i$ and $E_i^{\theta}(t)$ the event $\theta_i(t) \leq y_i$. Also, let $p_{i,t} = \Pr(\theta_1(t) \geq  y_i | \mathcal{F}_{t-1})$ where $\mathcal{F}_{t-1}$ is the history of plays until time $t$.
\end{definition}
\begin{definition}
	For two arms $i$ and $j$, if $\bar{\mu}_i(t) \neq \bar{\mu}_j(t)$, there exists a constant $\Delta_{ij}$ such that $|\bar{\mu}_i(t) - \bar{\mu}_j(t)| \geq \Delta_{ij}$. Let $\underline{\Delta} = \min \Delta_{ij}$.   
\end{definition}

We  have the following result on the frequency $m_i(T)$ of compensation the players receive for pulling each arm $i$ when considering the concentration inequality of Gaussian distribution \cite{abramowitz1965handbook}. 
\begin{lemma}{\label{lemma:3}}
	The expected frequency $\mathbb{E}(m_i(T))$ of compensation for pulling arm $i$ is bounded by $\frac{2\log T}{\underline{\Delta}^2}$. 
\end{lemma}

\begin{proof}
The proof is provided  in the supplemental material.
\end{proof}

Our analysis of regret generalizes that in \cite{agrawal2017near,feng2019intrinsic} to incorporate the effect of drift caused by compensation. 
\begin{theorem} \label{theroem:3}
	For the incentivized Thompson Sampling algorithm, the expected regret $R(T)$ and compensation $C(T)$ can be bounded as follows: 
	\begin{align}
	\mathbb{E}(R(T)) &\leq \sum_{i=2}^{K}((4e^{11} + 21)P_i(T) + \frac{5}{\Delta_i^2} + Q_i(T) + \frac{\pi^2}{6}), \\
	\mathbb{E}(C(T)) &\leq  2\max(l,1)K \frac{\log T}{\underline{\Delta}^2},
	\end{align} 
	where $P_i(T) = \frac{18 \log (T \Delta_i^2)}{\Delta_i^2}$ and $Q_i(T) =  \lceil \frac{9}{2\Delta_i^2}\left((1+ \frac{4\Delta_i l}{3\underline{\Delta}^2}) \log T + \sqrt{1+ \frac{8\Delta_i l \log T}{3\underline{\Delta}^2} }\right) \rceil$.
\end{theorem}
\begin{proof}
	The analysis of compensation is straightforward, similar to that for the incentivized  $\varepsilon$-Greedy algorithm. By Lemma \ref{lemma:3}, the expected compensation $\mathbb{E}(C(T)) \leq  2\max(l,1)K\frac{\log T}{\underline{\Delta}^2}$.
	
	Consider now the regret for choosing suboptimal arm $i(i\neq 1)$. We can bound $\mathbb{E}(n_i(T))$ as follows: 
	\begin{align*}
	&~~~~~\mathbb{E}(n_i(T)) \\
	&\leq \sum_{t=1}^{T} \Pr(I_t = i, E_i^\mu(t), E^{\theta}_i(t)) \\
	&~~~~+ \sum_{t=1}^{T} \Pr(I_t = i, E_i^\mu(t), \overline{E^{\theta}_i(t)}) + \sum_{t=1}^{T} \Pr(I_t = i, \overline{E^{\mu}_i(t)})
	\end{align*}
	The first two terms can be bounded by the results of \cite{agrawal2017near}, see the detail in supplemental material, since their analysis will not be affected by the reward drift. Specifically, by Lemma 4, the sum of first two terms is upper bounded by $cP_i(T) + \frac{5}{\Delta_i^2}$, where $c$ is certain constant and $P_i(T) =  \frac{18\log (T \Delta_i^2)}{\Delta_i^2}$. 
	As for the third term, the analysis is similar to that of UCB and $\varepsilon$-Greedy algorithm where the drift is bounded by $\mathcal{O}(\log T)$: 
	\small
	\begin{align*}
	&~~~~\sum_{t=1}^{T} \Pr(I_t = i, \overline{E^{\mu}_i(t)})\\
	&\leq \sum_{t=1}^{T} \Pr( \overline{E^{\mu}_i(t)}) =\sum_{t=1}^{T} \Pr( \bar{\mu}_i(t) \geq  x_i)  \\	
	& = \sum_{t=1}^{T} \Pr\left( \hat{\mu}_i(t) +  \frac{B_i(t)}{n_i(t)} \geq  x_i                 \right)	\\
	& = \sum_{t=1}^{T} \Pr\Big( \hat{\mu}_i(t) - \mu_i  \geq  \underbrace{\frac{\Delta_i}{3} - \frac{B_i(t)}{n_i(t)}}_{Y_i(t)} \Big)\\
	& \leq  \sum_{t=1}^{T} \Pr \left( \hat{\mu}_i(t) - \mu_i \geq Y_i(t), n_i(t) \leq Q_i\right) \\
	&~~~~+\sum_{t=1}^{T} \Pr \left( \hat{\mu}_i(t) - \mu_i \geq Y_i(t), n_i(t) \geq  Q_i\right) \\
	& \leq Q_i + \sum_{t=1}^{T} e^{-2n_i(t)Y_i(t)^2}  \leq Q_i + \frac{\pi^2}{6}.
	\end{align*}
	\normalsize
	where the second last inequality is due to Hoeffding's inequality.
	We then choose $Q_i$ such that, when $n_i(t) \geq Q_i$, 
	\begin{align}
	\frac{\Delta_i}{3} - \frac{B_i(t)}{n_i(t)} &\geq \frac{\Delta_i}{3} - \frac{2 l \log T }{\Delta^2 n_i(t)} \geq  0,  \label{equ:1}\\
	n_i(t) Y_i(t)^2  &\geq \log T. \label{equ:2}
	\end{align}
	
	By $\eqref{equ:1}, Q_i \geq \frac{6l \log T}{\Delta_i \underline{\Delta}^2}.$ 
	Since $n_i(t) Y_i(t)^2$ is non-increasing in $B_i(t)$, 
	equation \eqref{equ:2} requires
	\begin{align}
	\frac{\Delta_i^2}{9}n_i + \frac{4l^2 \log^2 T}{\underline{\Delta}^4} \frac{1}{n_i} \geq (1+\frac{4\Delta_i l}{3\underline{\Delta}^2}) \log T. \nonumber
	\end{align}
	The above two equations  lead to 
	$$Q_i \geq \lceil \frac{9}{2\Delta_i^2}\left((1+ \frac{4\Delta_i l}{3\underline{\Delta}^2}) \log T + \sqrt{1+ \frac{8\Delta_i l \log T}{3\underline{\Delta}^2} }\right) \rceil.$$
\end{proof}

\subsection{Discussion of Results}
As can be seen from the above analysis, all three instantiations of the incentivized exploration algorithm attain $\mathcal{O}(\log T)$ regret and compensation upper bound under drifted reward. Our results match both the theoretical lower bound for regret in \cite{lai1985asymptotically}  and lower bound for compensation in \cite{Wang2018MultiarmedBW}  without reward drift.  Although explicit lower bounds of the regret and compensation with drifted feedback in our setting remain unknown, we argue that these lower bounds should be larger or equal to the lower bound without reward drift since non-drifting environment is a special case of the drifted reward feedback with drift function $f_t = 0$. On the other hand, the proposed incentive mechanism is still cost-efficient even the payment will lead to biased feedback, as the principal can reduce the regret from $\mathcal{O}(T)$ for the players' myopic choices to $\mathcal{O}(\log T)$ by paying merely  $\mathcal{O}(\log T)$ in incentive. 




In terms of sensitivity to unknown drift functions $f_t$, both incentivized $\varepsilon$-Greedy and Thompson Sampling attain $\mathcal{O}(l)$ regret and compensation, while the incentivized UCB attains $\mathcal{O}((l+1)^2)$ regret and $\mathcal{O}(l+1)$ compensation. This difference comes from two aspects: 1) UCB is deterministic given the history while $\varepsilon$-Greedy and Thompson Sampling have a randomized exploration phase which makes them less sensitive to the drift. 2) For UCB, the drift effect is bounded by the amount of compensation which affects the frequency of compensation and in turn shapes the amount of compensation, while for $\varepsilon$-Greedy and Thompson Sampling, the cumulative drift can be directly bounded by the frequency of compensation. 
Numerical experiments reported in the next section are consistent with these analytical results.
\section{Numerical Examples}


In this section, we carry out numerical experiments using synthetic data to complement the previous analysis of the
incentivized MAB algorithms under reward drift, including UCB, $\varepsilon$-Greedy and Thompson Sampling.

We generate a pool of $K=9$ arms with mean reward vector $\bm{\mu} = [0.9,0.8,0.7,0.6,0.5,0.4,0.3,0.2,0.1]$. In each iteration, after the player pulls an arm $I_t$, reward $r_t$ is set to the arm’s mean reward plus a random term drawn from $\mathcal{N}(0, 1)$, i.e. $r_t = \bm{\mu}_{I_t}+\mathcal{N}(0, 1)$. Because of the randomness in sample rewards, the greedy algorithm without exploration suffers a linear regret, e.g., we observe nearly 6000 regret for 20000 trials. For the reward drift under compensation, we consider a linear drifting function  $b_t = l x_t$ where $x_t$ is the compensation offered by the principle and coefficient $l \geq 0$.  The player reveals drifted reward feedback $r_t+b_t$. 

\begin{figure}[ht!]
	\vspace{-3mm}
	\centering
	\includegraphics[width=1\columnwidth]{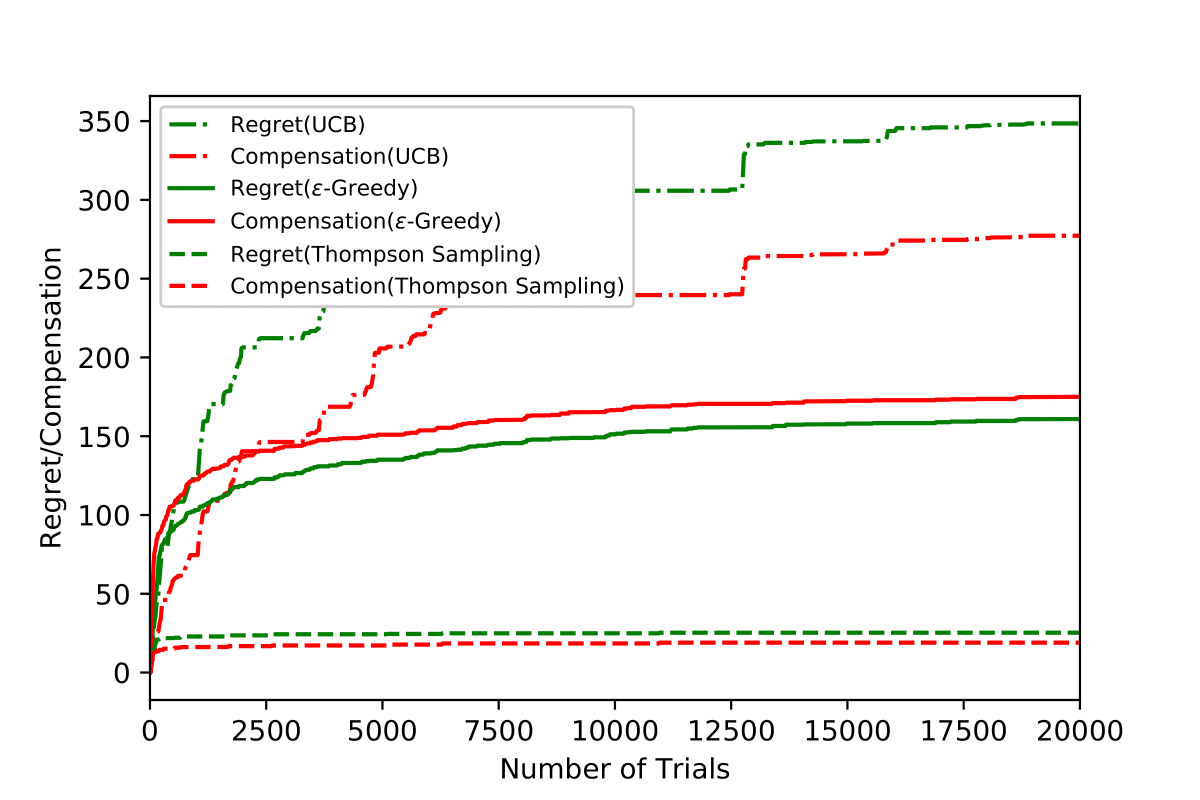}
	\caption{Regret and Compensation for UCB, $\varepsilon$-Greedy and Thompson Sampling without reward drift.
	}\label{fig:2} 
	\vspace{-3mm}
\end{figure}


\begin{figure}[ht!]
	\centering
	\includegraphics[width=1\columnwidth]{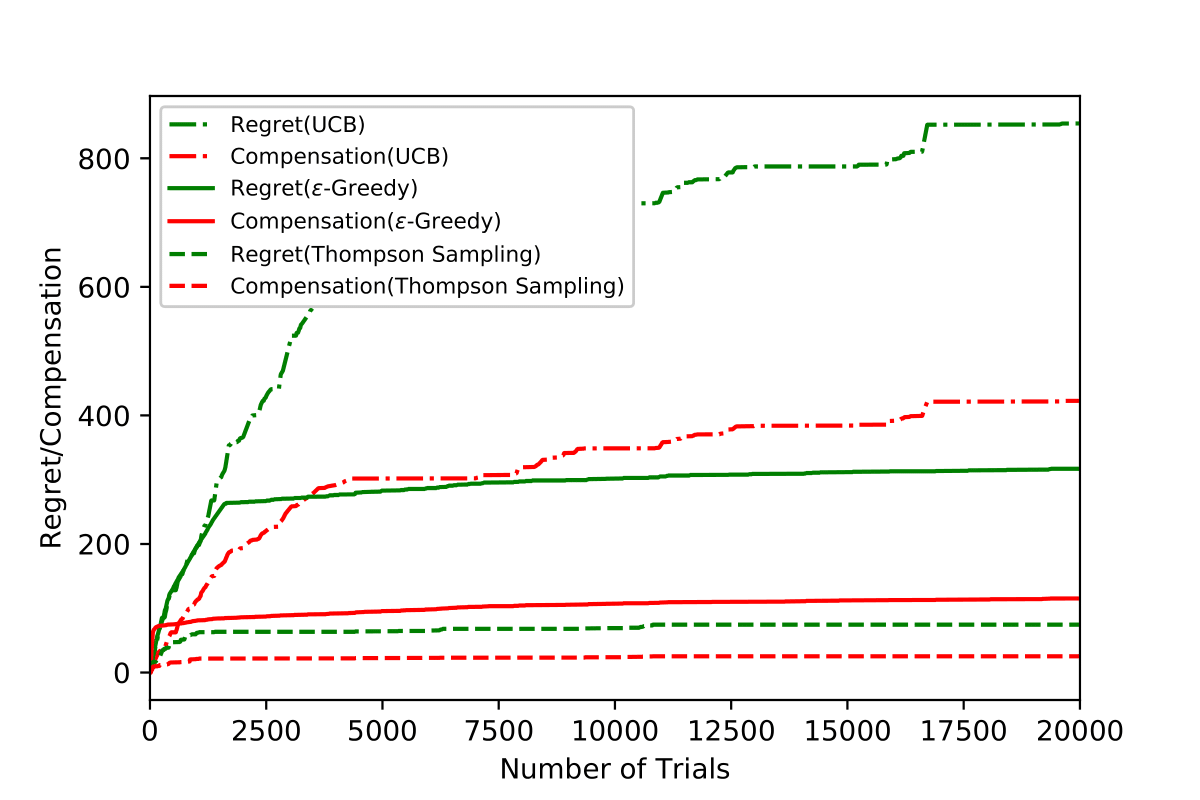}
	\caption{Regret and Compensation for UCB, $\varepsilon$-Greedy and Thompson Sampling with drift coefficient $l = 1.1$.
	}\label{fig:3} 
	\vspace{-3mm}
\end{figure}
For the incentivized exploration, we first compare regret and compensation in a non-drifting environment ($l=0$) and a drifted reward environment ($l>0$). In a non-drifting reward environment the player always gives unbiased feedback even offered with compensation. The result is shown in Fig. \ref{fig:2}. As expected, all three instantiations of the incentivized MAB algorithms have a sub-linear regret and compensation. Thompson Sampling outperforms the other two both in the regret (which is consistent with observation from previous work \cite{vermorel2005multi,chapelle2011empirical}) and compensation.

In Fig. \ref{fig:3} we show the performance of the incentivzed MAB algorithms under drifted reward with drift coefficient $l=1.1$.  We first observe that over the three algorithms Thompson Sampling still performs the best. While their relative performance are in same order as Fig. \ref{fig:2}, the regret and compensation are worse than non-drifting setting, e.g., regret of UCB increases from 350 to 800 because of the biased feedback.  

To better understand the effect of drifted reward, we vary the coefficient $l$ from $0$ to $1.1$ and present the results in Table \ref{table:1}.  We notice that the incentivized UCB incurs largest regret and compensation. This is due to the fact that, as the time goes, a larger UCB and uncertainty are assigned to those arms that are less explored but may in fact have small mean rewards, and the resulting higher chance of those suboptimal arms being selected leads to larger regret and compensation. We also notice that the gap between regret and compensation of UCB increase faster compared to the other two. This is consistent with out theoretical analysis that  the regret of UCB is in the order of $\mathcal{O}((l+1)^2)$ and compensation is in the order  $\mathcal{O}{(l+1)}$.



\begin{table} \small 
	\begin{center}
		\begin{tabular}{|@{\hspace{0.5\tabcolsep}} c @{\hspace{0.5\tabcolsep}} @{\hspace{0.5\tabcolsep}} c @{\hspace{0.5\tabcolsep}}@{\hspace{0.5\tabcolsep}}c@{\hspace{0.5\tabcolsep}} @{\hspace{0.5\tabcolsep}}c@{\hspace{0.5\tabcolsep}}@{\hspace{0.5\tabcolsep}}c@{\hspace{0.5\tabcolsep}}@{\hspace{0.5\tabcolsep}}c@{\hspace{0.5\tabcolsep}}@{\hspace{0.5\tabcolsep}}c@{\hspace{0.5\tabcolsep}}@{\hspace{0.5\tabcolsep}}c@{\hspace{0.5\tabcolsep}}|} 
			\hline  
			$l$ & 0 & 0.05 & 0.1& 0.4&0.7&0.9& 1.1 \\  
			\hline
			UCB(R) &  348.5&  432.1&451.9 &522.8&615.1&712.9&854.2 \\ 
			UCB(C) &277.2  & 292.9 &349.5 &375.6&408.0&473.0&422.7  \\
			\hline 
			$\varepsilon$-Greedy(R) & 160.0 & 170.3 &218.0 &260.1&266.2&272.6&317.0  \\
			$\varepsilon$-Greedy(C) & 185.9 & 217.4 &130.4 &167.6&102.8&161.8& 115.2 \\
			\hline
			TS(R) & 25.3 & 28.2 &33.4 &37.1&46.3&63.6&74.5 \\
			TS(C) & 18.9  &23.7  &20.9 &29.3&22.9&29.1&25.3  \\
			\hline
		\end{tabular}
		\caption{Regret (R) and Compensation (C) with different drift coefficients.} \label{table:1}	
	\end{center}
\end{table}

\begin{table} \small
	\begin{center}
		\begin{tabular}{|@{\hspace{0.5\tabcolsep}} c @{\hspace{0.5\tabcolsep}} @{\hspace{0.5\tabcolsep}} c @{\hspace{0.5\tabcolsep}}@{\hspace{0.5\tabcolsep}}c@{\hspace{0.5\tabcolsep}} @{\hspace{0.5\tabcolsep}}c@{\hspace{0.5\tabcolsep}}@{\hspace{0.5\tabcolsep}}c@{\hspace{0.5\tabcolsep}}@{\hspace{0.5\tabcolsep}}c@{\hspace{0.5\tabcolsep}}@{\hspace{0.5\tabcolsep}}c@{\hspace{0.5\tabcolsep}}@{\hspace{0.5\tabcolsep}}c@{\hspace{0.5\tabcolsep}}|} 
			\hline  
			$l$ & 0 & 0.05 & 0.1& 0.4&0.7&0.9& 1.1 \\  
			\hline
			UCB(N) &  1225&  1639&1954 &2172&2288&2912&3374 \\ 
			UCB(E) &0.4\%  & 0.9\% &0.5\% &1.2\%&1.9\%&0.4\%&3.1\%  \\
			\hline 
			$\varepsilon$-Greedy(N) & 273 & 329 &304 &303&276&293&308  \\
			$\varepsilon$-Greedy(E) & 0.7\% & 1.5\% &0.5\% &1.0\%&1.6\%&0.4\%& 0.8\% \\
			\hline
			TS(N) & 60 & 79 &58 &98&131&109&106 \\
			TS(E) & 0.7\%  &0.7\%  &1.6\% &2.0\%&0.1\%&1.7\%&0.7\%  \\
			\hline
		\end{tabular}
		\caption{Number of compensation (N) and relative error (E) of estimation of arm $1$ with different drift coefficients.} \label{table:2}	
	\end{center}
	\vspace{-3mm}
\end{table}

We then exam the frequency of compensation, as well as the estimation error for arm 1 in terms of the relative error of the average drifted reward compared to the mean reward, and present the result in Table \ref{table:2}. We see that all three incentivized exploration algorithms achieve small estimation errors that are not sensitive to the drift coefficient $l$. This is expected, as the the expected compensation and thus the drift per time approaches 0 as T increases. However, while the incentivized $\varepsilon$-Greedy and Thompson Sampling have roughly a constant frequency of compensation across different $l$ values, the incentivized UCB is more sensitive to the coefficient in the frequency of compensation. The constant frequency of compensation for $\varepsilon$-Greedy and Thompson Sampling can be seen from the proof of Theorem \ref{theroem:2} and Lemma \ref{lemma:3} that show the frequency does not depend on the drift. In contrast, seen from the proof of Theorem \ref{theroem:1}, the frequency of compensation for UCB depends on the drift through equation (\ref{eq:egrb}). 


	\section{Conclusion}
	We propose and study multi-armed bandit algorithm with incentivized exploration under reward drift, where the player provides a biased reward feedback that is the sum of the true reward and a drift term that is non-decreasing in compensation. We analyze the regret and compensation for three instantiations of the incentivized MAB algorithm where the principal employs UCB, $\epsilon$-Greedy and Thompson Sampling, respectively.  Our results show that the algorithms achieve $\mathcal{O}(\log T)$ regret and compensation, and are therefor effective in incentivizing exploration. Our current analysis is based on the assumption that the reward drift is non-decreasing over the compensation. In the future work, we would like to study other assumptions about drift function and their corresponding impact on regret and compensation. It is also important to explore if an algorithm can leverage the drifted reward to reduce the compensation.

	\bibliography{reference}
	\bibliographystyle{aaai}

		\section{Supplementary Material}
\begin{fact}{(Hoeffding's inequality \cite{hoeffding1994probability})} \label{fact:1}
	Assume that $X_1,\cdots, X_n$ are i.i.d drawn from any distribution with mean $\mu$ and support $[0,1]$. Let $\bar{X} = \frac{1}{n} \sum_{i=1}^{n} X_i$, then for any $\delta \geq  0$, 
	\begin{align*}
	\Pr(\bar{X} - \mu \geq \delta) \leq  e^{-2n\delta^2}.
	\end{align*}
\end{fact}
\begin{fact}{\cite{abramowitz1965handbook}} \label{fact:2}
	For a Gaussian distributed random variable $X$ with mean $\mu$ and variance $\sigma^2$, for any x,
	\begin{equation*}
	\Pr(|X-\mu| \geq x\sigma) \leq \frac{1}{2}e^{-x^2/2}.
	\end{equation*} 
\end{fact}
\begin{lemma}\label{lemma:1}{(Lemma 4 in \cite{jun2018adversarial} and Lemma 9 in \cite{agarwal2014taming})}
	Let $\delta \leq 1/2$ and suppose $t$ satisfy $\sum_{\tau=1}^{t}\varepsilon_\tau \geq \frac{K}{e-2} \log (K/\delta).$ With probability at least $1-\delta$, the number of exploration $m_i(t)$ of arm $i$ up to time $t$ is bounded as follows: 
	\begin{align}
	m_i(t) \leq \sum_{\tau = 1}^{t} \frac{\varepsilon_\tau}{K} + \sqrt{3\sum_{\tau=1}^{t} \frac{\varepsilon_\tau}{K} \log \frac{K}{\delta}}.	
	\end{align}
\end{lemma} 
\begin{lemma}\label{lemma:2}{(Theorem 3 in \cite{auer2002finite} and Theorem 3.3 in \cite{feng2019intrinsic})}
	For the $\varepsilon$-Greedy algorithm with $\varepsilon = \min\{1,cK/t\}$, for any arm $i \in [K], t > K$, we have
	\begin{align*}
	\Pr \left( |\hat{\mu}_i(t) - \mu_{i}| \geq  \frac{\Delta_i}{3} \right) &\leq x_t e^{-x_t/5} + \frac{18}{\Delta_i^2}e^{-\frac{\Delta_i^2 \lfloor x_t \rfloor}{18}},\\
	\Pr \left( |\hat{\mu}_1(t) - \mu_{1}| \geq  \frac{\Delta_i}{3} \right) &\leq x_t e^{-x_t/5} + \frac{18}{\Delta_i^2}e^{-\frac{\Delta_i^2 \lfloor x_t \rfloor}{18}},
	\end{align*}
	where $x_t = \frac{1}{2K} \sum_{t=1}^{t} \varepsilon_\tau$. 
	If  $c \geq \max \frac{36}{\Delta_i}$, the sum of this probability up to $T$ has the following upper bound:
	\begin{align*}
	(\frac{c}{2} + \frac{18}{\Delta_i^2}) \log T + c(K+\frac{\pi^2}{\Delta_i^2}) + \frac{18}{\Delta_i^2}.
	\end{align*}  
	
\end{lemma}
\begin{lemma}{(lemma 2.14,2.16 in \cite{agrawal2017near})}\label{lemma:4}
	For the Thompson Sampling algorithm, if we choose $x_i = \mu_{i} + \frac{\Delta_i}{3}, y_i = \mu_{1} - \frac{\Delta_1}{3}$, then 
	\begin{align*}
	\sum_{t=1}^{T} \Pr(I_t = i, E_i^\mu(t), E^{\theta}_i(t)) &\leq (4e^{11} + 20)P_i(T) + \frac{4}{\Delta_i^2}, \\
	\sum_{t=1}^{T} \Pr(I_t = i, E_i^\mu(t), \overline{E^{\theta}_i(t)}) &\leq P_i(T) + \frac{1}{\Delta_i^2},	
	\end{align*}
	where $P_i(T) = \frac{18\log (T \Delta_i^2)}{\Delta_i^2}$.
\end{lemma}

\noindent\textbf{Details of inequality \eqref{bound:B}}
    \begin{align}
	B_i(t) &= \sum_{\tau =1}^{t} b_\tau \mathbb{I}(E_\tau^i = 1) \nonumber\\
	&\leq \sum_{\tau =1}^{t} l_\tau \sqrt{\frac{2\log \tau}{n_{i}(\tau)}}\mathbb{I}(E_\tau^i = 1) \nonumber\\
	&\leq \sum_{\tau =1}^{t} l \sqrt{\frac{2\log t}{n_{i}(\tau)}}\mathbb{I}(E_\tau^i = 1) \nonumber\\
	&= \sum_{m =1 }^{n_i(t)} l \sqrt{\frac{2\log t}{m}}  \leq 2l \sqrt{2n_i(t) \log t}, \nonumber
	\end{align}	
	where the last inequality is due to $\sum_{m=1}^{n} \frac{1}{\sqrt{m}} \leq 2\sqrt{n}$.
\begin{proof}[\bf Proof of Lemma \ref{lemma:3}]
	A compensation will be incurred for choosing arm $i$ at time $t$ only when
         \begin{align*}
	i \neq \arg \max_j \bar{\mu}_j(t) ~ \text{and} ~ i = \arg \max_j \theta_j(t)
	\end{align*}
	with  $\theta_j(t) \sim \mathcal{N}(\bar{\mu}_j(t),\frac{1}{n_j(t)+1})$. This can happen because distributions of different arms overlap, and in particular, there exists at least one arm $j$ such that $\theta_i(t) \geq \theta_j(t)$ while $\bar{\mu}_i(t) \leq \bar{\mu}_j(t) $. We then have a bound on the probability of compensation:  	
	\begin{align*}
	&~~~\Pr(I_t \neq G_t, I_t = i|\mathcal{F}_{t-1}) \\
	&\leq \Pr\left(\theta_i(t) \geq \theta_j(t),\bar{\mu}_i(t) \leq \bar{\mu}_j(t) | \mathcal{F}_{t-1}\right).
	\end{align*}
	Since $\theta_i(t)$ and $\theta_j(t)$ are independent samples from $\mathcal{N}(\bar{\mu}_i(t),\frac{1}{n_i(t)+1})$ and $\mathcal{N}(\bar{\mu}_k(t),\frac{1}{n_j(t)+1})$, the difference $Z = \theta_i - \theta_j$ follows a Gaussian distribution too:  
	\begin{align*}
	Z \sim \mathcal{N}(\bar{\mu}_i(t) - \bar{\mu}_j(t), \frac{1}{n_i(t) + 1} +\frac{1}{n_j(t)+1}).
	\end{align*}
	For simplicity of presentation, let $\bar{\mu}_{ij} =  \bar{\mu}_i(t) - \bar{\mu}_j(t)$ and $\sigma_{ij} = \sqrt{\frac{1}{n_i(t) + 1} +\frac{1}{n_j(t)+1}}$. 
	We have
	\begin{align*}
	&~~~\Pr\left(\theta_i \geq \theta_j,\bar{\mu}_i(t) \leq \bar{\mu}_j(t) | \mathcal{F}_{t-1}\right) \\
	&= \Pr(Z \geq 0, Z \sim \mathcal{N}(\bar{\mu}_{ij}, \sigma_{ij}^2)|\mathcal{F}_{t-1}) \\
	&= \Pr(Z-\bar{\mu}_{ij} \geq -\bar{\mu}_{ij},Z \sim \mathcal{N}(\bar{\mu}_{ij}, \sigma_{ij}^2)  |\mathcal{F}_{t-1}) \\
	&= \Pr(Z - \bar{\mu}_{ij} \geq -\frac{\bar{\mu}_{ij}}{\sigma_{ij}} \sigma_{ij},Z \sim \mathcal{N}(\bar{\mu}_{ij}, \sigma_{ij}^2)  |\mathcal{F}_{t-1}) \\ 
	& \leq \frac{1}{2} e^{-\frac{\bar{\mu}_{ij}^2}{2\sigma_{ij}^2}} \leq \frac{1}{2} e^{-\frac{\underline{\Delta}^2}{2\sigma_{ij}^2}} \leq \frac{1}{2} e^{-\frac{\underline{\Delta}^2 (n_i(t)+1)}{2}}, 
	\end{align*}
	where the first inequality is due to the concentration inequality for Gaussian distribution (Fact \ref{fact:2} in supplementary material). 
	Let $L = \frac{2\log T}{\underline{\Delta}^2} - 1$. When $n_i(t) \geq L$, it is straightforward to verify that $e^{-\frac{\underline{\Delta}^2 (n_i(t)+1)}{2}} \leq \frac{1}{T}$. The expected frequency of compensation for pulling arm $i$ can be bounded as follows: 
	\begin{align*}
	\mathbb{E}(m_i) &=\sum_{t=1}^{T} \Pr(I_t \neq G_t, I_t = i|\mathcal{F}_{t-1}) \\
	&\leq  \sum_{t=1}^{T} \Pr(I_t \neq G_t, I_t = i, n_i(t) \leq L|\mathcal{F}_{t-1})\\
	&~~~~~ + \sum_{t=1}^{T} \Pr(I_t \neq G_t, I_t = i, n_i(t) \geq L|\mathcal{F}_{t-1}) \\ 
	& \leq L + \sum_{t=1}^{T} \frac{1}{T} =\frac{2\log T}{\underline{\Delta}^2}. 
	\end{align*}
\end{proof}
\end{document}